\documentclass{article}
\usepackage{microtype}
\usepackage{graphicx}
\usepackage{subcaption}
\usepackage{booktabs}
\usepackage{iclr2020_conference, times}
\iclrfinalcopy
\usepackage{hyperref}
\graphicspath{{figures/}}
\usepackage{amsmath,amsthm,amssymb}
\usepackage{thmtools, thm-restate}
\newtheorem{theorem}{Theorem}
\newtheorem{lemma}{Lemma}
\usepackage{bm}
\newcommand{\vect}[1]{\boldsymbol{\mathbf{#1}}}

\usepackage[ruled,vlined]{algorithm2e}

\begin{document}

\title{Stochastic gradient algorithms from ODE splitting perspective}

\author{Daniil Merkulov \& Ivan Oseledets \\
Center for Computational and Data-Intensive Science and Engineering\\
Skolkovo Institute of Science and Technology\\
Bolshoy Boulevard 30, bld. 1, Moscow, Russia, 121205 \\
\texttt{daniil.merkulov@skolkovotech.ru, i.oseledets@skoltech.ru}
}

\maketitle

\begin{abstract}
We present a different view on stochastic optimization, which goes back to the splitting schemes for approximate solutions of ODE. In this work, we provide a connection between stochastic gradient descent approach and first-order splitting scheme for ODE. We consider the special case of splitting, which is inspired by machine learning applications and derive a new upper bound on the global splitting error for it. We present, that the Kaczmarz method is the limit case of the splitting scheme for the unit batch SGD for linear least squares problem. We support our findings with systematic empirical studies, which demonstrates, that a more accurate solution of local problems leads to the stepsize robustness and provides better convergence in time and iterations on the softmax regression problem.
\end{abstract}

\section{Introduction}
A lot of practical problems arising in machine learning require minimization of a finite sample average which can be written in the form
\begin{equation}\label{strang:finitesum}
    f(\vect{\vect{\theta}}) = \frac{1}{n} \sum_{i=1}^n f_i(\vect{\vect{\theta}}) \rightarrow \min_{\vect{\vect{\theta}} \in \mathbb{R}^p},
\end{equation}
where the sum goes over the \emph{minibatches} of the original dataset. Vanilla stochastic gradient descent
(SGD) method \cite{robbins1951stochastic} consists sequential steps in the direction of the gradient of $f_i(\vect{\theta})$, where $i$ is to be chosen randomly from $1$ to $n$ without replacement.
\begin{equation}
\vect{\theta}_{k+1} = \vect{\theta}_{k} - h_{k} \nabla f_i.
\end{equation}
Gradient descent method \cite{cauchy1847methode} can be considered as an Euler discretization of the ordinary
differential equation (ODE) of the form of the gradient flow
\begin{equation}\label{strang:euler}
    \frac{d \vect{\theta}}{d t} = -\nabla f(\vect{\theta}).
\end{equation}
In continuous time, SGD if often analyzed by introducing noise into the right-hand side of \eqref{strang:euler}. However, for a real dataset, the distribution of the noise obtained by replacing the full gradient by its minibatch variant is not known and can be different for different problems. Instead, we propose a new view on the SGD as a \emph{first-order splitting scheme} for \eqref{strang:euler}, thus shedding a new light on SGD-type algorithms. This representation allows using more efficient local problem solvers for the approximation of the full gradient flow.

\textbf{Contributions}
\begin{itemize}
    \item We show, that vanilla SGD could be considered as a splitting scheme for a full gradient flow and highlight connection between learning rate, batch size and size of the approximation step of SGD in continuous time.
    \item  We propose new optimization scheme, which uses numerical integration of simple ODE at each step instead of stochastic gradient calculation and show empirically, that such approach can be considered as a stepsize-robust alternative to SGD for some practical ML problems.
    \item We present, that the Kaczmarz method is the limit case of the splitting scheme for the unit batch SGD for linear least squares problem.
\end{itemize}


\section{SGD as a splitting scheme}

We firstly consider simple ODE, where we can apply splitting idea and corresponding minimization problem. The best example to start from is simple ODE with right-hand-side, consisting of two summands:
\begin{equation}
    \frac{d \vect{\theta}}{d t} = - \frac{1}{2} \left( g_1(\vect{\theta}) + g_2(\vect{\theta})\right)
    \label{strang:gradientflow}
\end{equation}
Suppose, we want to find the solution $\vect{\theta}(h)$ of \eqref{strang:gradientflow} via integrating it on the small timestep $h$. The first order splitting scheme defined by solving first $\frac{d \vect{\theta}}{d t} = - \frac{1}{2} g_1(\vect{\theta}), \quad \vect{\theta}(0) = \vect{\theta}_0$ with exact solution $\vect{\theta}_1(h)$ at the moment $h$, followed by $\frac{d \vect{\theta}}{d t} = - \frac{1}{2} g_2(\vect{\theta}), \quad \vect{\theta}(0) = \vect{\theta}_1(h)$ with exact solution $\vect{\theta}_2(h)$ at the moment $h$. Thus, the first order approximation could be written as a combinations of both solutions $\vect{\theta}^I(h) = \vect{\theta}_2(h) \circ \vect{\theta}_1(h) \circ \vect{\theta}_0$.

It is interesting to study how the pure splitting scheme \cite{marchuk1968some,strang1968construction} corresponds to the SGD approach. For this purpose, we consider an illustrative example of Gradient Flow equation \ref{strang:simple_GF}, where the right-hand side of ODE is just the sum of operators acting on $\vect{\theta}$, which allows us to apply splitting scheme approximation directly.
\begin{equation}
\label{strang:simple_GF}
\frac{d \vect{\theta}}{d t} = -\frac{1}{2} \sum\limits_{i=1}^2 \nabla f_i (\vect{\theta}) = - \frac{1}{2} \nabla f_1 (\vect{\theta}) -\frac{1}{2}  \nabla f_2 (\vect{\theta})
\end{equation}
\begin{table}[h!]
\caption{The table describes the correspondence between splitting scheme for discretized Gradient Flow ODE and epoch of SGD}
\resizebox{\textwidth}{!}{
\begin{tabular}{cccc}
\toprule
\textbf{Splitting step} & \textbf{Euler discretization} & \textbf{SGD Epoch} & \textbf{First-order splitting} \\
\midrule
$\frac{d \vect{\theta}}{d t} = -\frac{1}{2}\nabla f_1(\vect{\theta})$ & $\tilde{\vect{\theta}}_{I} = \vect{\theta}_0 - \frac{h}{2}\nabla f_1 (\vect{\theta}_0) $&$\tilde{\vect{\theta}}_{SGD} = \vect{\theta}_0 - h \nabla f_1 (\vect{\theta}_0) $&$\tilde{\vect{\theta}}_{I} = \vect{\theta}_0 - \frac{h}{2}\nabla f_1 (\vect{\theta}_0)$ \\
$\frac{d \vect{\theta}}{d t} = -\frac{1}{2}\nabla f_2(\vect{\theta}) $&$\vect{\theta}_{I} = \tilde{\vect{\theta}}_{I} - \frac{h}{2}\nabla f_2 (\tilde{\vect{\theta}}_{I}) $&$\vect{\theta}_{SGD} = \tilde{\vect{\theta}}_{SGD} - h \nabla f_2 (\tilde{\vect{\theta}}_{SGD}) $&$\vect{\theta}_{I} = \tilde{\vect{\theta}}_{I} - \frac{h}{2}\nabla f_2 (\tilde{\vect{\theta}}_{I})$ \\
\bottomrule
\end{tabular}
}
\end{table}

Thus, we can conclude, that \textit{one epoch of SGD is just the splitting scheme for the discretized Gradient Flow ODE with $2 \cdot h$ step size ($m \cdot h$ in case of $m$ batches)}

Indeed, in SGD we go in the direction of the batch gradient, which stands for the Euler discretization of batch gradient flow ODE or \emph{local ODE}. This idea gives additional intuition on the method. Given information about the Euler scheme limitation (first-order accuracy, stability issues), we propose to solve each local problem more precisely.

\section{Optimization step with ODE solver}

We propose to integrate local problem more precisely instead of Euler step in SGD. Solution of the local ODE problem involves replacing gradient in the right-hand side of gradient flow ODE \ref{strang:gradientflow} with batch gradient version. In our experiments the explicit Runge-Kutta method \cite{dormand1980family,shampine1986some} was used via scipy \cite{2020SciPy-NMeth} function odeint.

\begin{table}[h!]
\caption{The table presents ODE, which we need to solve at each step of the algorithm. The last column shows the ODE, which is needed to be solved at each iteration of the algorithm for each given problem.}
\resizebox{\textwidth}{!}{
\begin{tabular}{cccc}
\toprule
\textbf{Problem} & \textbf{Loss function} & \textbf{Batch gradient} & \textbf{Initial local ODE} \\
\midrule
Linear Least Squares & $f(\vect{\theta}) = \frac{1}{n}\sum\limits_{i=1}^m\Vert X_i \vect{\theta} - \vect{y_i} \Vert_2^2$ & $\frac{1}{b}X_i^\top( X_i \vect{\theta} - \vect{y_i})$ & $\frac{d \vect{\theta}}{d t} = - \frac{1}{n} X_i^\top( X_i \vect{\theta} - \vect{y_i})$ \\
Binary logistic regression & $\begin{aligned}f(\vect{\theta}) = -\frac{1}{n} \sum_{i=1}^n\left(y_i \ln \sigma(\vect{\theta}^\top\vect{x_i})  \right.&+ \\ \left.+ (1-y_i) \ln \left(1-\sigma(\vect{\theta}^\top\vect{x_i})\right)\right)&\end{aligned}$ & $\frac{1}{b}X_i^\top\left( \sigma\left(X_i \vect{\theta}\right) - \vect{y_i}\right)$ & $\frac{d \vect{\theta}}{d t} = - \frac{1}{n} X_i^\top\left( \sigma\left(X_i \vect{\theta}\right) - \vect{y_i}\right)$ \\
One FC Layer + softmax & $f(\Theta) = 
-\frac{1}{n} \sum\limits_{i=1}^n\log\left(\frac{\vect{y_i}^\top e^{\Theta^\top \vect{x_i}}}{\vect{1}^\top e^{\Theta^\top \vect{x_i}}}\right)$ & $ \frac{1}{b} X_i^\top\left(s(\Theta^\top X_i^\top) - Y_i \right)^\top$ & $\frac{d \Theta}{d t} = - \frac{1}{n} X_i^\top\left(s(\Theta^\top X_i^\top) - Y_i \right)^\top$\\
\bottomrule
\end{tabular}
}
\end{table}

\begin{algorithm}[h!]
\SetAlgoLined
$\vect{\theta}_0$ - initial parameter; $b$ - batch size; $\alpha$ - learning rate; $m$- total number of batches

$h := \alpha m$

$t := 0$

\For{$k = 0,1, \ldots$}{
    \For{$i = 1,2, \ldots, m$}{
        Formulate local ODE problem $\mathcal{P}_i^k$

        $\vect{\theta}_{t+1} = $ integrate $\mathcal{P}_i^k$ given an initial value $\vect{\theta}(0) = \vect{\theta}_t$ to the step h

        $t := t+1$ 
        }
    }
\caption{Splitting optimization}
\end{algorithm}

Typical machine learning problems involves dealing with mini-batch of size $b$, which is often less, than the number of trainable parameters $p$, which allows us to reduce dimensionality of the dynamic system via $QR$ decomposition of each batch data matrix $X_i^\top = Q_i R_i$ (see details in the Appendix) and substitution $\vect{\eta}_i = Q_i^\top \vect{\theta}$. Note, that $QR$ decomposition is only needed to be performed once before the training.

\begin{table}[h!]
\caption{The table shows initial local ODE and paired $\mathcal{P}_i^k$. Note, that $\vect{\eta}_i \in \mathbb{R}^b$ , while $\vect{\theta} \in \mathbb{R}^p$}
\resizebox{\textwidth}{!}{
\begin{tabular}{llc}
\toprule
\multicolumn{1}{c}{\textbf{Initial local ODE}} & \multicolumn{1}{c}{$\mathcal{P}_i^k$} & \textbf{Integration} \\
\midrule
$\frac{d \vect{\theta}}{d t} = - \frac{1}{n} X_i^\top( X_i \vect{\theta} - \vect{y_i})$ & $\frac{d \vect{\eta_i}}{d t} = - \frac{1}{n} R_i\left(R_i^\top \vect{\eta_i} - \vect{y_i}\right), \vect{\eta_i} = Q_i^\top\vect{\theta}$ & analytical \\
$\frac{d \vect{\theta}}{d t} = - \frac{1}{n} X_i^\top\left( \sigma\left(X_i \vect{\theta}\right) - \vect{y_i}\right)$ & $\frac{d \vect{\eta_i}}{d t} = - \frac{1}{n} R_i\left(\sigma\left(R_i^\top \vect{\eta_i}\right) - \vect{y_i}\right), \vect{\eta_i} = Q_i^\top\vect{\theta}$ & \texttt{odeint} \\
$\frac{d \Theta}{d t} = - \frac{1}{n} X_i^\top\left(s(\Theta^\top X_i^\top) - Y_i \right)^\top$ & $\frac{d H_i}{dt} = - \frac{1}{n} R_i(s(H_i^\top R) - Y_i)^\top, H_i = Q_i^\top \Theta $& \texttt{odeint} \\
\bottomrule
\end{tabular}
}
\end{table}

There is an analytical solution for each local ODE in linear least squares case:

\begin{restatable}{theorem}{llsls}\label{strang:LLS_local_solution} For any matrix $\vect{x_i} \in \mathbb{R}^{b \times p}, b \leq p, \text{rank}X_i = b$, any vector of right-hand side $\vect{y_i} \in \mathbb{R}^{b}$ and initial vector of parameters $\vect{\theta}_0$, there is a solution of the $\frac{d \vect{\theta}}{d t} = - \frac{1}{n} X_i^\top( X_i \vect{\theta} - \vect{y_i})$, given by formula:
\begin{equation}\label{strang:LLS_local_solution_formula}
\vect{\theta}(h) = Q_i e^{-\frac{1}{n}R_iR_i^\top h} \left( Q_i^\top \vect{\theta}_0 - R_i^{-\top}\vect{y_i}\right) + Q_iR_i^{-\top}\vect{y_i} + (I - Q_iQ_i^\top)\vect{\theta}_0,
\end{equation}
where $Q_i \in \mathbb{R}^{p \times b}$ and $R_i \in \mathbb{R}^{b \times b}$ stands for the $QR$ decomposition of the matrix $\vect{X_i}^\top$, $\vect{X_i}^\top = Q_i R_i$.
\end{restatable}

It is interesting to mention, that the splitting approach immediately leads to the Kaczmarz \cite{kaczmarz1937method, strohmer2009randomized, gower2015randomized} method for solving linear system in the same setting with unit batch size.

\begin{equation}
\label{strang:splitting_limit_kaczmarz}
\lim_{h \to \infty} \vect{\theta}(h) = \frac{\left(y_i - \vect{x_i}^\top\vect{\theta}_0 \right)}{\|\vect{x_i}\|^2} \vect{x_i} + \vect{\theta}_0,
\end{equation}

which is exact formula for Kaczmarz method for solving linear system. This result correlates with the statements of \cite{needell2014stochastic}, but provides us with a new sense of similarity between SGD and Kaczmarz method.


\section{Results}

In this section, we describe the experimental setting. The majority of computations were performed on the NVIDIA DGX-2 cluster with 80 CPUs and 512 Gb RAM. We restricted the number of CPU usage per each experiment with an upper limit of 5 CPUs per experiment. All time measurements were done with the time library for Python. All experiments were done with the fixed random seed for reproducibility. For each experiment we performed 30 runs with random initialization and plotted trend line with the standard deviation.

\textbf{Linear Least Squares}  Both random and the real linear systems were tested. For random linear system (\texttt{random lls}) we generated $10000\times500$ matrix with additive Gaussian noise of magnitude $0.01$. Presented figures correspond to the batch size equals to $20$. The real linear system (\texttt{tom lls}) is the standard tomography data from AIRTools II \cite{hansen2018air}. Solution of the linear system is the $50 \times 50$ image reconstructed from solving $12780 \times 2500$ linear system. Presented figures correspond to the batch size equals to $60$. Relative error $10^{-3}$ was used as the stopping criterion.

\textbf{Binary Logistic Regression}  (\texttt{logreg}) In our experiments we used two classes from MNIST \cite{lecun1998gradient} dataset, which corresponds to the $0$ and $1$ digits. The size of the batch for presented figure is $50$. Test error $0.001$ was used as the stopping criterion.

\textbf{Softmax Logistic Regression} (\texttt{softmax}) We took Fashion MNIST \cite{xiao2017fashion} dataset with $60000$ grayscale pictures from $10$ classes. Each example is $28 \times 28$ image. The size of the batch for presented figure is $64$. Test error $0.25$ was used as the stopping criterion.

On the figures below we have two labels: \texttt{SGD} and \texttt{Splitting}, which stands for batch stochastic gradient descent and proposed algorithm. We use different constant learning rates to perform our experiments. All the learning rates tested for both algorithms. Lack of point of one algorithm on the graph means reaching the limit of iterations without achieving the termination rule.

\begin{figure}[h!]
    \begin{subfigure}[b]{0.25\textwidth}
            \centering
            \includegraphics[width=\linewidth]{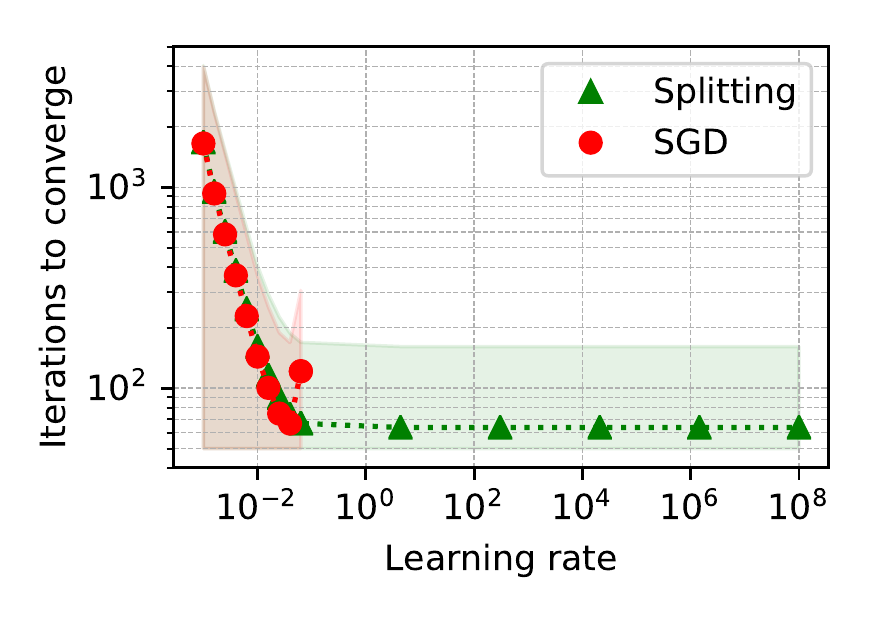}
            \caption{{\small \texttt{Random LLS}}}
    \end{subfigure}%
    \begin{subfigure}[b]{0.25\textwidth}
            \centering
            \includegraphics[width=\linewidth]{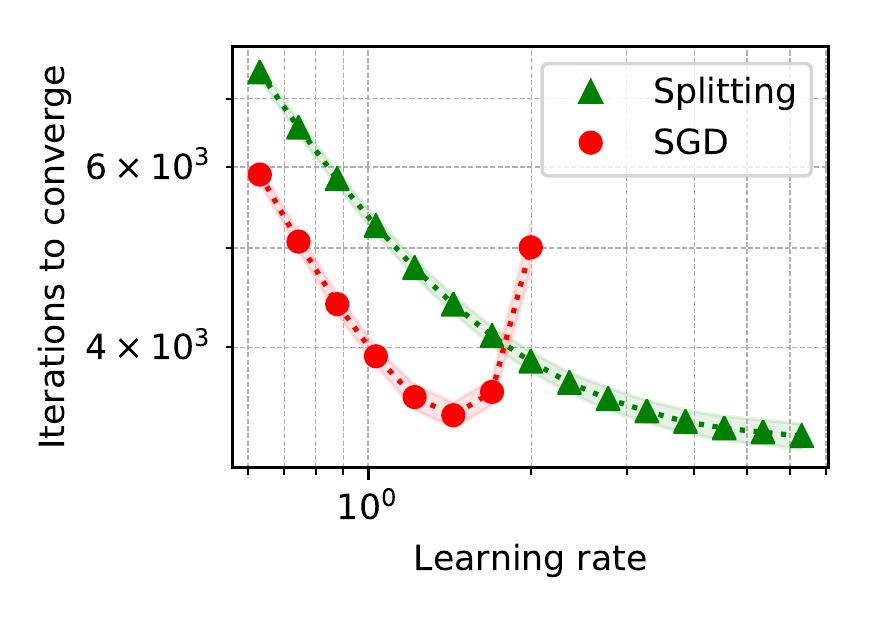}
            \caption{{\small \texttt{Tom LLS}}}
    \end{subfigure}%
    \begin{subfigure}[b]{0.25\textwidth}
            \centering
            \includegraphics[width=\linewidth]{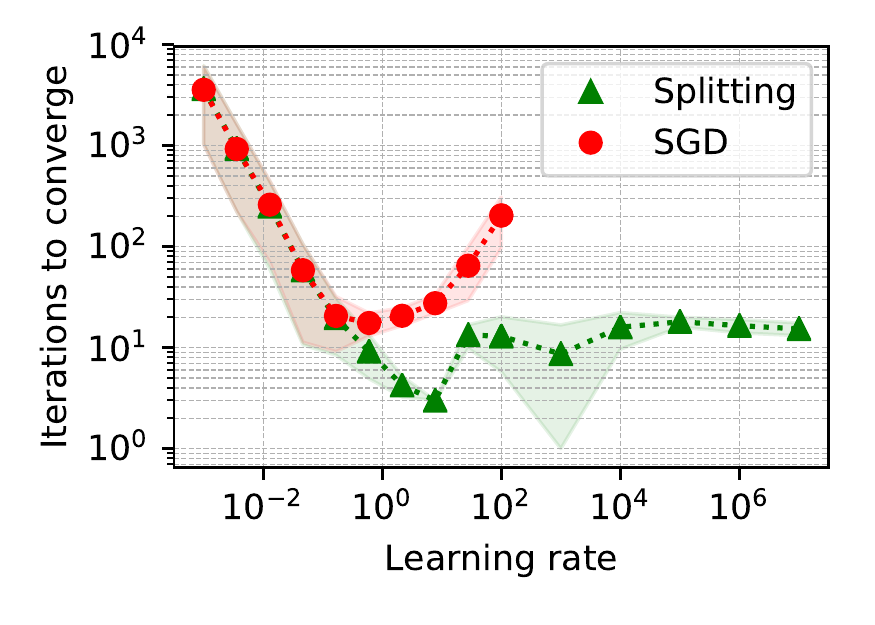}
            \caption{{\small \texttt{LogReg}}}
    \end{subfigure}%
    \begin{subfigure}[b]{0.25\textwidth}
            \centering
            \includegraphics[width=\linewidth]{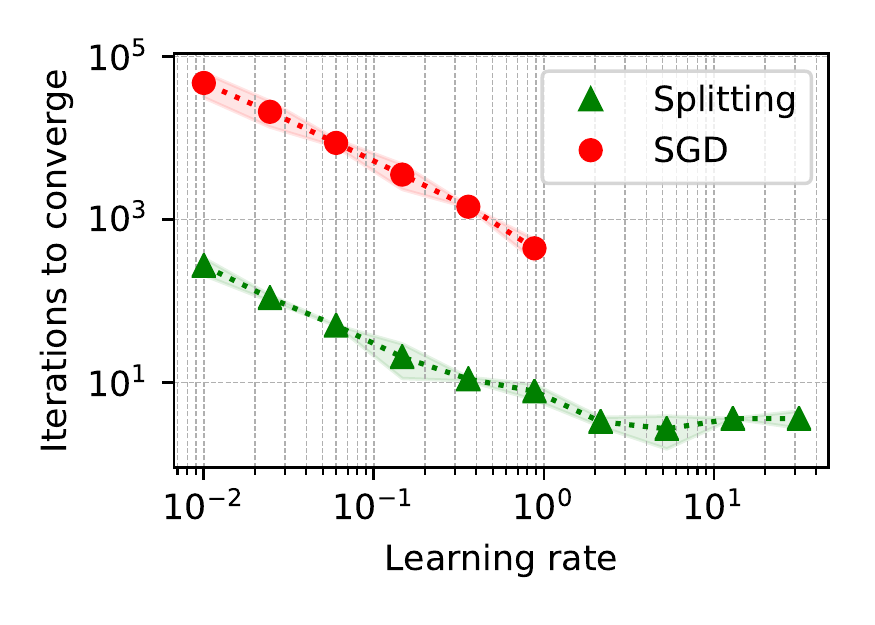}
            \caption{{\small \texttt{Softmax}}}
    \end{subfigure}
\end{figure}

As it is expected, SGD diverges starting from some value of learning rate, which is specific for each problem. While we can see comparative robustness of the proposed splitting optimization approach.

\begin{figure}[h!]
    \begin{subfigure}[b]{0.25\textwidth}
            \centering
            \includegraphics[width=\linewidth]{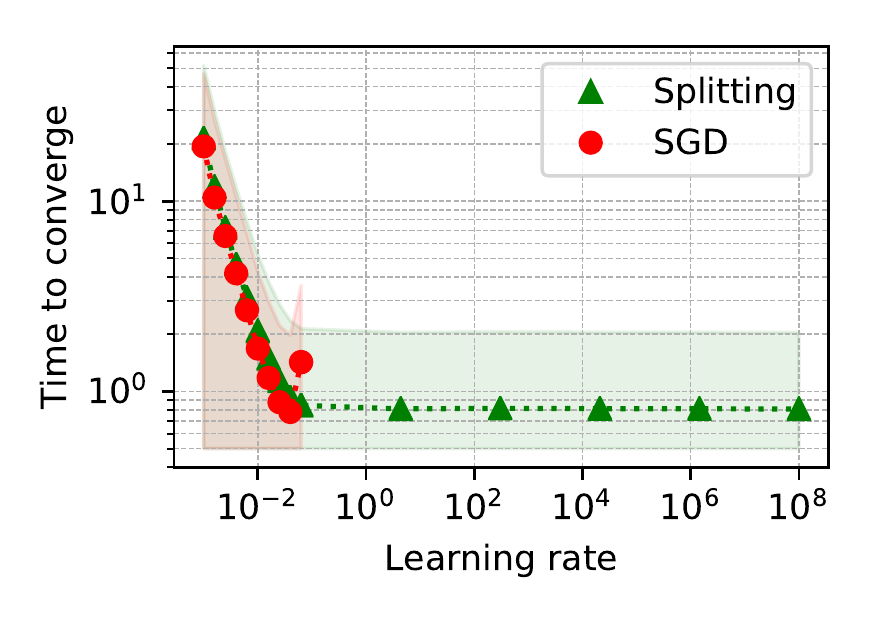}
            \caption{{\small \texttt{Random LLS}}}
    \end{subfigure}%
    \begin{subfigure}[b]{0.25\textwidth}
            \centering
            \includegraphics[width=\linewidth]{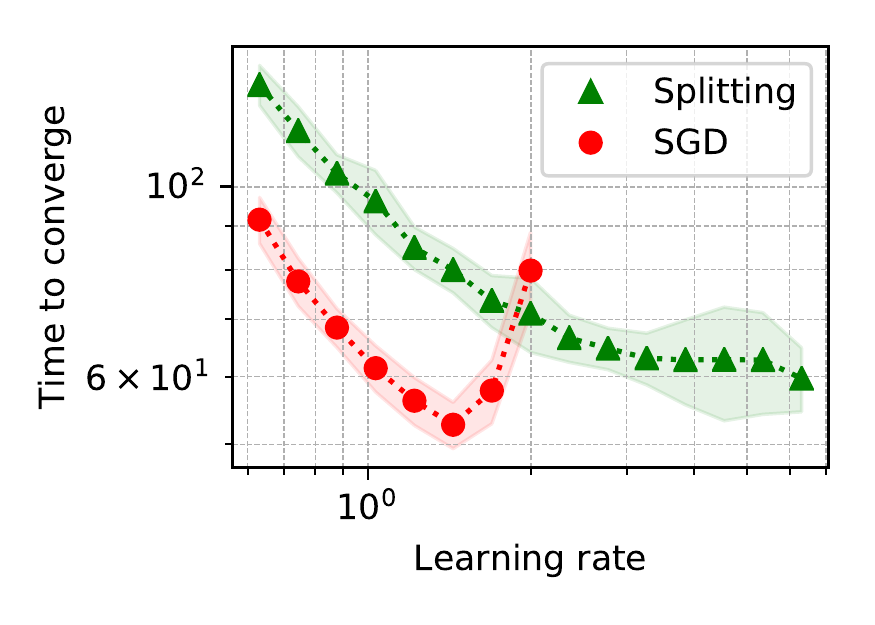}
            \caption{{\small \texttt{Tom LLS}}}
    \end{subfigure}%
    \begin{subfigure}[b]{0.25\textwidth}
            \centering
            \includegraphics[width=\linewidth]{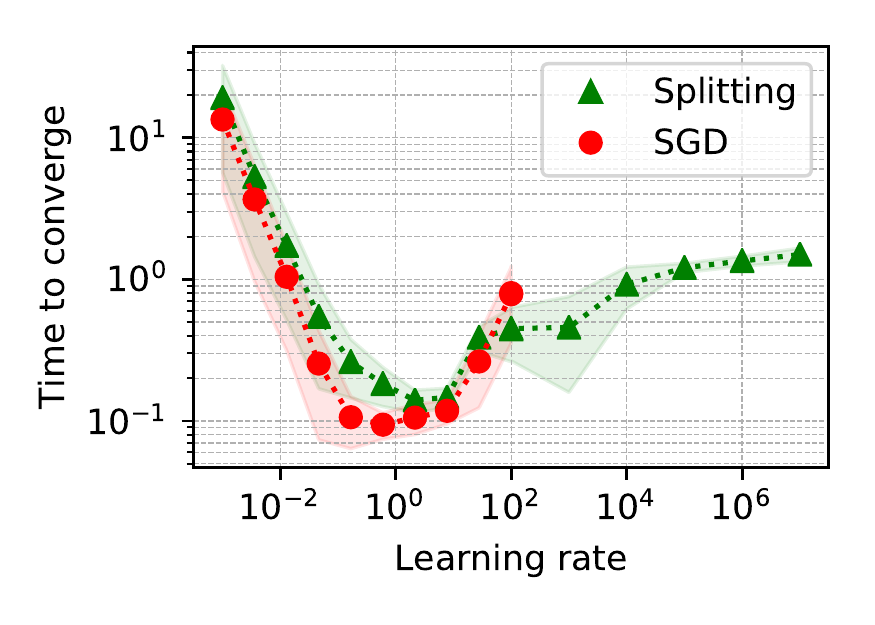}
            \caption{{\small \texttt{LogReg}}}
    \end{subfigure}%
    \begin{subfigure}[b]{0.25\textwidth}
            \centering
            \includegraphics[width=\linewidth]{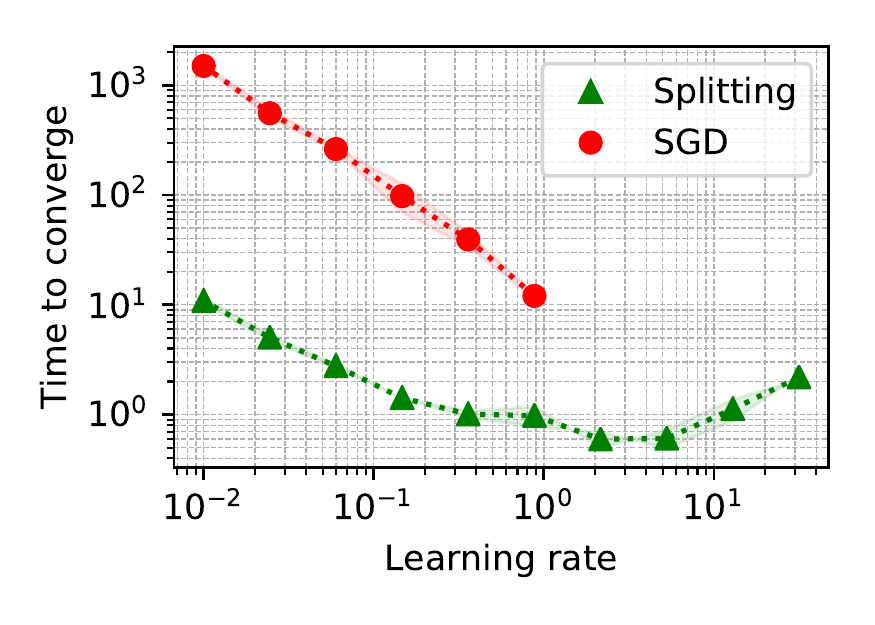}
            \caption{{\small \texttt{Softmax}}}
    \end{subfigure}
\end{figure}


\section{Related work}
In this work, we presented another point of view on the nature of stochasticity in the stochastic gradient algorithms. From this perspective, different splitting schemes yield different stochastic gradient algorithms. We focused on the first-order splitting scheme for ODE, which corresponds to the SGD with the constant learning rate. Given this tractable setting, we performed a systematic empirical study of the local problem integration influence on the quality of the approximation scheme in machine learning problems. While the question of using these ideas to make general-purpose optimizer remains open, splitting optimization approach showed itself quite robust to the hyperparameter tuning for particular practical problems. Appendix to the paper contains proofs of the theorems and a new global error upper bounds for the first-order splitting for the special case.
In \cite{su2014differential}  authors introduced second order ODE, which is equivalent (in the limit sense) to the gradient descent with Nesterov momentum \cite{nesterov1983method}.  Generalization of these ideas was presented in \cite{wibisono2016variational} with an arbitrary polynomial acceleration using the same parameter in ODE. General overview of the interplay between continuous-time and discrete-time points of view on dynamical systems and iterative optimization methods is covered in \cite{helmke2012optimization}, \cite{evtushenko1994stable}


\newpage
{ \small
\bibliography{biblio}}
\bibliographystyle{iclr2020_conference}


\appendix


\section{Upper bound on the global splitting error}

Suppose, that we have only two batches, and the problem \eqref{strang:LLS} is consistent, i.e. there exists an exact solution $\vect{\theta}_*$ such as $X \vect{\theta}_* = \vect{y}$. The GD flow has the form
\begin{equation}\label{strang:model1}
\begin{split}
    \frac{d \vect{\theta}}{d t} &= -X^{\top} (X \vect{\theta} - \vect{y}) = -X^{\top} X(\vect{\theta} - \vect{\theta}_*) =\\ &= -(X_1^{\top} X_1 + X^{\top}_2 X_2)(\vect{\theta} - \vect{\theta}_*),
\end{split}
\end{equation}
i.e. the splitting scheme corresponds to a linear operator splitting

\begin{equation*}
A = A_1 + A_2, \; A = -X^{\top} X, \; A_i = -X^{\top}_i X_i, \; i = 1, 2.
\end{equation*}

Both $A_1$ and $A_2$ are symmetric non-negative definite matrices. Without loss of generality, we can assume that $\vect{\theta}_* = 0$,

Suppose that the rank of $A$ is $r_1$ and the rank of $A_2$ is $r_2$. Then, we can write them as
\begin{equation*}
A_i = Q_i B_i Q^*_i,
\end{equation*}
where $Q_i$ is an $N \times r_i$ matrix with orthonormal columns.  The following Lemma gives the representation of the matrix exponents of such matrices.
\begin{lemma}\label{strang:lemexp}
Let $A = Q B Q^*,$ where $Q$ is an $N \times r$ matrix with orthonormal columns, and $B$ is an $r \times r$ matrix. Then,
\begin{equation}\label{strang:lrexp}
    e^{t A}  = (I - QQ^*) + Q e^{t B} Q^*.
\end{equation}
\end{lemma}
To prove \eqref{strang:lrexp} we note that
\begin{equation*}
\begin{split}
e^{t A} &= \sum_{k=0}^{\infty} \frac{t^k A^k}{k!} = \sum_{k=0}^{\infty} \frac{t^k Q B^k Q^*}{k!} = \\ &= I - QQ^* + QQ^* + Q \sum_{k=1}^{\infty} \frac{t^k B^k}{k!} Q^* = \\ &= (I - QQ^*) + Q e^{t B} Q^*.
\end{split}
\end{equation*}

\begin{lemma}
    \label{strang:lemupper_2}
    Let $A_1, A_2 \in \mathbb{S}^p_{+}$ be the square negative semidefinite matrices, that don't have full rank, i.e. $\operatorname{rank}{A_1} \leq p$ and $\operatorname{rank}{A_2} \leq p$. While the sum of those matrices has full rank, i.e. $A = A_1 + A_2, \operatorname{rank}{A} = p$. Then, the global upper bound error will be written as follows:

    \begin{equation}\label{strang:lemupper}
        \lim_{t \to \infty}\| e^{A_2t}e^{A_1t} - e^{At}\| = \|(I - Q_2Q_2^*)(I - Q_1Q_1^*)\|
    \end{equation}
\end{lemma}
\begin{proof}
    The proof is straightforward. We will use the low rank matrix exponential decomposition from the Lemma \ref{strang:lemexp}
    $$
    e^{A_it} = \Pi_i + Q_i e^{B_it} Q_i^*, \text{where } \Pi_i = I - Q_iQ_i^*; i = 1,2
    $$
    \begin{align*}
    &\lim_{t \to \infty}\| e^{A_2t}e^{A_1t} - e^{At}\| = \\ 
    &= \lim_{t \to \infty}\| (\Pi_2 + Q_2 e^{B_2t} Q_2^*)(\Pi_1 + Q_1 e^{B_1t} Q_1^*) - e^{At}\| = \\
    &= \lim_{t \to \infty}\| \Pi_2\Pi_1 + Q_1 e^{B_1t} Q_1^*\Pi_2 + \Pi_1Q_2 e^{B_2t} Q_2^* + \\
    &+  Q_1 e^{B_1t} Q_1^* Q_2 e^{B_2t} Q_2^* - e^{At}\| =\\&= \Pi_2 \Pi_1
    \end{align*}
    Since all matrices $B_1, B_2, A$ are negative all the matrix exponentials are decaying: $\|e^{At}\|\leq e^{t\mu (A)}\, \forall t\geq 0$, where $\mu(A) = \lambda_{max} \left( \frac{A + A^\top}{2}\right)$ - the logarithmic norm.
\end{proof}

\begin{figure}[h!]
    \begin{subfigure}[t]{0.49\textwidth}
            \centering
            \includegraphics[width=\linewidth]{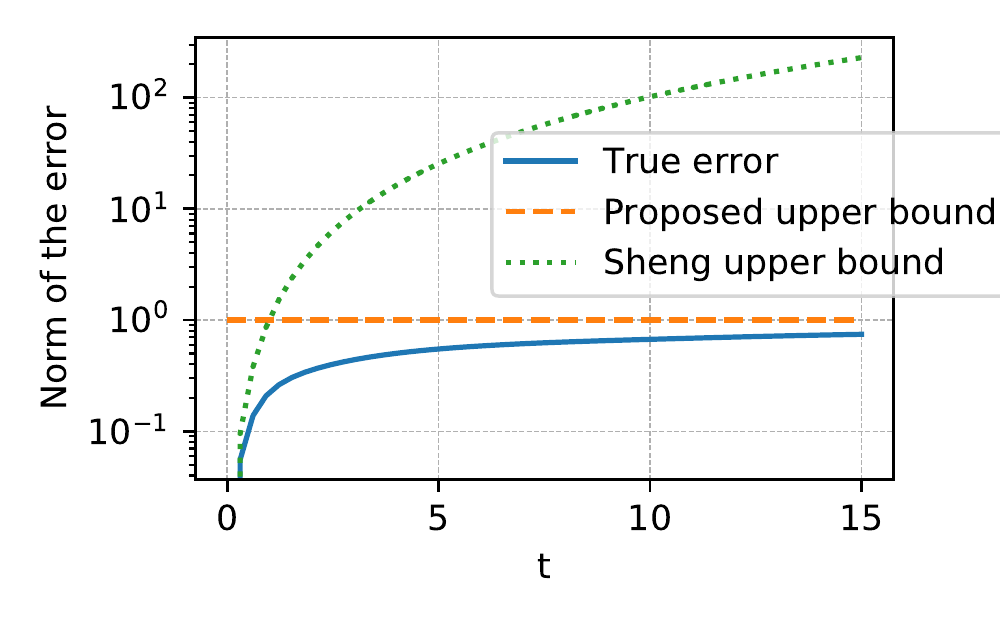}
            \caption{{\small Global error of the splitting scheme. Initial random full rank matrix $X \in \mathbb{R}^{100 \times 100}$ was splitted by rows. $X_1, X_2 \in \mathbb{R}^{50 \times 100}$. Target matrices were obtained the following way: $A_1 = -X_1^*X_1, A_2 = -X_2^*X_2, A = -X^*X$. So $A_1, A_2$ are negative and lacking full rank, while $A = A_1 + A_2$ has full rank.}}
            \label{strang:fig:upper_bound_2}
    \end{subfigure}%
    \hfill
    \begin{subfigure}[t]{0.49\textwidth}
            \centering
            \includegraphics[width=\linewidth]{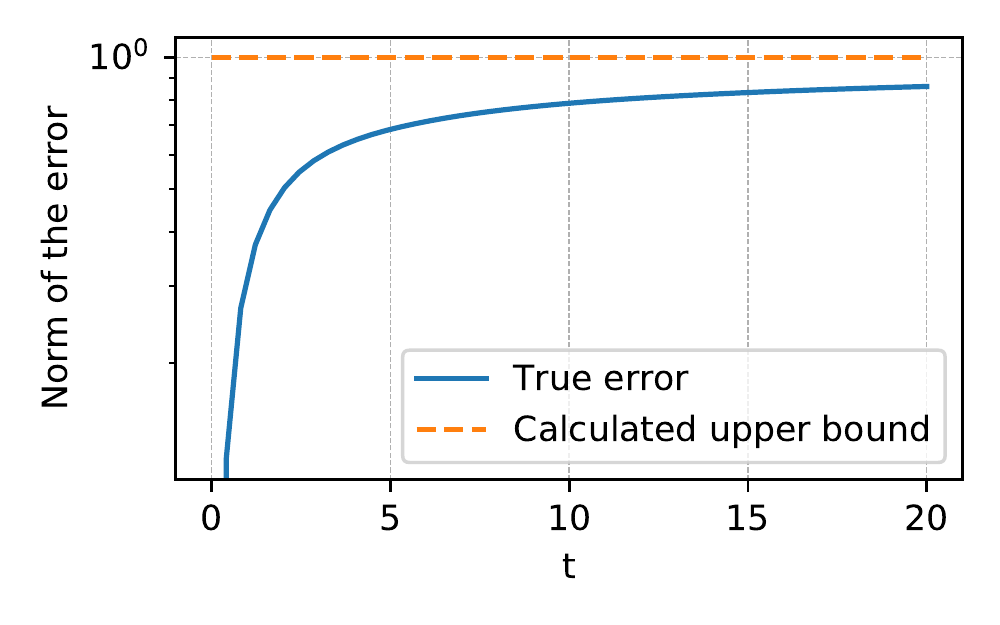}
            \caption{{\small Global upper bound on the splitting scheme in case of $40$ summands in the right-hand side.}}
            \label{strang:fig:upper_bound_many}
    \end{subfigure}
\end{figure}

The graph presented on the Figure \ref{strang:fig:upper_bound_2} describes . One can easily see significant difference between existing global upper bounds for that case \cite{sheng1994global} and derived upper bound.

\begin{theorem}\label{strang:theorem_uppbound}
    Let $A_1, A_2, \ldots, A_b \in \mathbb{S}^p_{+}$ be the square negative semidefinite matrices, that don't have full rank, i.e. $\operatorname{rank}{A_i} \leq p, \;\forall i = 1, \ldots, b$. While the sum of those matrices has full rank, i.e. $A = \sum\limits_{i=1}^b A_i, \operatorname{rank}{A} = p$. Then, the global upper bound error will be written as follows:

    \begin{equation}\label{strang:global_error_upper_bound}
        \lim_{t \to \infty}\| e^{A_bt} \cdot \ldots \cdot e^{A_1t} - e^{At}\| = \left\|\prod\limits_{i=1}^b \Pi_{b-i+1}\right\|,
    \end{equation}
    where $\Pi_i = I - Q_iQ_i^*$ and $A_i = Q_iB_iQ_i^*$ and $Q_i$ is a matrix with orthonormal columns. 
\end{theorem}

The graph on the Figure \ref{strang:fig:upper_bound_many} shows empirical validity of the presented upper bound.


\section{Proofs}

\llsls*

\begin{proof}
Given $X_i^\top = Q_i R_i$, we have $(I - Q_iQ_i^\top)X_i^\top = 0$. Note, that $Q_i$ is left unitary matrix, i.e. $Q_i^\top Q_i = I$. 
\begin{align}
\nonumber \frac{d \vect{\theta}}{d t} &= - \frac{1}{n}X_i^\top(X_i \vect{\theta} - \vect{y_i}) \\
\nonumber (I - Q_iQ_i^\top)\frac{d \vect{\theta}}{d t} &= 0 \\
\nonumber \frac{d \vect{\theta}}{d t} &= Q_i\frac{d (Q_i^\top\vect{\theta})}{d t} \quad Q_i^\top \vect{\theta} = \vect{\eta_i}\\ 
\nonumber \frac{d \vect{\theta}}{d t} &= Q_i\frac{d \vect{\eta_i}}{d t} \quad \text{integrate from $0$ to $h$}\\ 
\label{strang:lls_theorem_theta_from_eta}\vect{\theta}(h) &= Q_i \left(\vect{\eta_i}(h) - \vect{\eta_i}(0) \right) + \vect{\theta}_0
\end{align}

On the other hand:
\begin{align}\nonumber
\frac{d \vect{\eta_i}}{d t} &= Q_i^\top\frac{d \vect{\theta}}{d t} =  - \frac{1}{n} Q_i^\top X_i^\top(X_i \vect{\theta} - \vect{y_i}) = \\ 
\nonumber&= - \frac{1}{n} Q_i^\top  Q_i R_i( R_i^\top Q_i^\top \vect{\theta} - \vect{y_i}) =\\
&= - \frac{1}{n} \left( R_iR_i^\top \vect{\eta_i} - R_i \vect{y_i}\right) \label{strang:lls_theorem_eta_from_theta}
\end{align}

Consider the moment of time $t = \infty$. $\frac{d \vect{\eta_i}}{d t} = 0$, since $\exists \vect{\theta}^*, Q_i^\top \vect{\theta}^* = \vect{\eta_i}^*$. Also consider \eqref{strang:lls_theorem_eta_from_theta}:

\begin{equation}\label{strang:lls_theorem_eta_star}
\begin{split}
\frac{d \vect{\eta_i}}{d t} = 0 &= - \frac{1}{n} \left( R_iR_i^\top \vect{\eta_i}^* - R_i \vect{y_i}\right) \\ R_i \vect{y_i} &= R_iR_i^\top \vect{\eta_i}^*
\end{split}
\end{equation}

Now we look at the \eqref{strang:lls_theorem_eta_from_theta} with the replacement, given in \eqref{strang:lls_theorem_eta_star}:

\begin{align*}
\frac{d \vect{\eta_i}}{d t} &= - \frac{1}{n} \left( R_iR_i^\top \vect{\eta_i}- R_iR_i^\top \vect{\eta_i}^*\right) \\
\frac{d \vect{\eta_i}}{d t} &= - \frac{1}{n} R_iR_i^\top \left( \vect{\eta_i}- \vect{\eta_i}^*\right) \qquad \text{integrate from $0$ to $h$} \\
\vect{\eta_i}(h) - \vect{\eta_i}^* &= e^{- \frac{1}{n} R_iR_i^\top h} (\vect{\eta_i}(0) - \vect{\eta_i}^*) \\ 
&\text{while } \vect{\eta_i}^* = R_i^{-\top}\vect{y_i}, \vect{\eta_i}(0) = Q_i^\top \vect{\theta}_0 \\
\vect{\eta_i}(h) &= e^{- \frac{1}{n} R_iR_i^\top h} (Q_i^\top \vect{\theta}_0 - R_i^{-\top}\vect{y_i}) + R_i^{-\top}\vect{y_i} 
\end{align*}

Using \eqref{strang:lls_theorem_theta_from_eta} we obtain the target formula
\begin{equation*}
\begin{split}
\vect{\theta}(h) &= Q_i e^{-\frac{1}{n}R_iR_i^\top h} \left( Q_i^\top \vect{\theta}_0 - R_i^{-\top}\vect{y_i}\right) + \\ &+ Q_iR_i^{-\top}\vect{y_i} + (I - Q_iQ_i^\top)\vect{\theta}_0,
\end{split}
\end{equation*}

\end{proof}

\begin{lemma}\label{strang:lemexp}
Let $A = Q B Q^*,$ where $Q$ is an $N \times r$ matrix with orthonormal columns, and $B$ is an $r \times r$ matrix. Then,
\begin{equation}\label{strang:lrexp}
    e^{t A}  = (I - QQ^*) + Q e^{t B} Q^*.
\end{equation}
\end{lemma}
To prove \eqref{strang:lrexp} we note that
\begin{equation*}
\begin{split}
e^{t A} &= \sum_{k=0}^{\infty} \frac{t^k A^k}{k!} = \sum_{k=0}^{\infty} \frac{t^k Q B^k Q^*}{k!} = \\ &= I - QQ^* + QQ^* + Q \sum_{k=1}^{\infty} \frac{t^k B^k}{k!} Q^* = \\ &= (I - QQ^*) + Q e^{t B} Q^*.
\end{split}
\end{equation*}

\begin{lemma}
    \label{strang:lemupper_2}
    Let $A_1, A_2 \in \mathbb{S}^p_{+}$ be the square negative semidefinite matrices, that don't have full rank, i.e. $\operatorname{rank}{A_1} \leq p$ and $\operatorname{rank}{A_2} \leq p$. While the sum of those matrices has full rank, i.e. $A = A_1 + A_2, \operatorname{rank}{A} = p$. Then, the global upper bound error will be written as follows:

    \begin{equation}\label{strang:lemupper}
        \lim_{t \to \infty}\| e^{A_2t}e^{A_1t} - e^{At}\| = \|(I - Q_2Q_2^*)(I - Q_1Q_1^*)\|
    \end{equation}
\end{lemma}
\begin{proof}
    The proof is straightforward. We will use the low rank matrix exponential decomposition from the Lemma \ref{strang:lemexp}
    $$
    e^{A_it} = \Pi_i + Q_i e^{B_it} Q_i^*, \text{where } \Pi_i = I - Q_iQ_i^*; i = 1,2
    $$
    \begin{align*}
    &\lim_{t \to \infty}\| e^{A_2t}e^{A_1t} - e^{At}\| = \\ 
    &= \lim_{t \to \infty}\| (\Pi_2 + Q_2 e^{B_2t} Q_2^*)(\Pi_1 + Q_1 e^{B_1t} Q_1^*) - e^{At}\| = \\
    &= \lim_{t \to \infty}\| \Pi_2\Pi_1 + Q_1 e^{B_1t} Q_1^*\Pi_2 + \Pi_1Q_2 e^{B_2t} Q_2^* + \\
    &+  Q_1 e^{B_1t} Q_1^* Q_2 e^{B_2t} Q_2^* - e^{At}\| =\\&= \Pi_2 \Pi_1
    \end{align*}
    Since all matrices $B_1, B_2, A$ are negative all the matrix exponentials are decaying: $\|e^{At}\|\leq e^{t\mu (A)}\, \forall t\geq 0$, where $\mu(A) = \lambda_{max} \left( \frac{A + A^\top}{2}\right)$ - the logarithmic norm.
\end{proof}


\section{Applications}
\subsection{Linear least squares}
\subsubsection{Problem}
Let $f_i(\vect{\theta}) = \Vert \vect{x_i}^{\top} \vect{\theta} - y_i \Vert^2$, then problem \eqref{strang:finitesum} is the linear least squares problem, which can be written as
\begin{equation}\label{strang:LLS}
   f(\vect{\theta}) = \frac{1}{n}\Vert X \vect{\theta} - \vect{y} \Vert_2^2  = \frac{1}{n}\sum\limits_{i=1}^s\Vert X_i \vect{\theta} - \vect{y_i} \Vert_2^2\to \min_{\vect{\theta} \in \mathbb{R}^p},
\end{equation}
where $X \in \mathbb{R}^{n \times p}$ and $\vect{y} \in \mathbb{R}^p$ and the second part of the equation stands for $s$ mini-batches with size $b$ regrouping ($b \cdot s = n$): $X_i \in \mathbb{R}^{b \times p}, \vect{y_i} \in \mathbb{R}^{b}$

\begin{equation}\label{strang:LLS_grad}
\nabla_\theta f(\vect{\theta}) = \nabla f(\vect{\theta}) = \frac{1}{n}\sum\limits_{i=1}^s X_i^\top(X_i \vect{\theta} - \vect{y_i})
\end{equation}

The gradient flow equation will be written as follows:
\begin{equation}\label{strang:LLS_GF}
\frac{d \vect{\theta}}{d t} = - \frac{1}{n}\sum\limits_{i=1}^s X_i^\top( X_i \vect{\theta} - \vect{y_i})
\end{equation}

\subsubsection{Exact solution of the local problem}
Theorem \ref{strang:LLS_local_solution} gives us explicit formula for the local solution:
$$
\vect{\theta}(h) = Q_i e^{-\frac{1}{n}R_iR_i^\top h} \left( Q_i^\top \vect{\theta}_0 - R_i^{-\top}\vect{y_i}\right) + Q_iR_i^{-\top}\vect{y_i} + (I - Q_iQ_i^\top)\vect{\theta}_0
$$

\subsubsection{Kaczmarz as the limit case of splitting}
Kaczmarz method \cite{kaczmarz1937method}, \cite{strohmer2009randomized}, \cite{gower2015randomized} is a well-known iterative algorithm for solving linear systems
It is interesting to mention, that splitting approach immediately leads to the Kaczmarz method for solving linear system in the same setting with unit batch size.

When the batch size is equal to one, we need to do $n$ QR decompositions for each transposed batch matrix, which is just column vector $\vect{x_i}$ in our case:

\begin{equation}
\vect{x_i} = \vect{q_i} \vect{r_i} = \underset{\vect{q_i} }{\frac{\vect{x_i}}{\|\vect{x_i}\|}} \underset{\vect{r_i}}{\vphantom{\frac{\vect{x_i}}{\|\vect{x_i}\|}} \|\vect{x_i}\|}
\end{equation}

Now, we need to use \eqref{strang:LLS_local_solution_formula} to derive analytic local solution in that case:

\begin{equation*}
\begin{split}
\vect{\theta}(h) &= \frac{\vect{x_i}}{\|\vect{x_i}\|} e^{-\frac{\|\vect{x_i}\|^2 h}{n}} \left( \frac{\vect{x_i}^\top}{\|\vect{x_i}\|} \vect{\theta}_0 - \frac{y_i}{\|\vect{x_i}\|}\right) + \\ &+ \frac{\vect{x_i}}{\|\vect{x_i}\|^2}y_i + \left(I - \frac{\vect{x_i}\vect{x_i}^\top}{\|\vect{x_i}\|^2}\right)\vect{\theta}_0 = \\
&= \frac{\left(y_i -\vect{x_i}^\top\vect{\theta}_0 \right)}{\|\vect{x_i}\|^2}  \left(1 - e^{-\frac{\|\vect{x_i}\|^2 h}{n}}\right)\vect{x_i} + \vect{\theta}_0
\end{split}
\end{equation*}

It can be easily seen, that:

\begin{equation}
\label{strang:splitting_limit_kaczmarz}
\lim_{h \to \infty} \vect{\theta}(h) = \frac{\left(y_i - \vect{x_i}^\top\vect{\theta}_0 \right)}{\|\vect{x_i}\|^2} \vect{x_i} + \vect{\theta}_0,
\end{equation}
which is exact formula for Kaczmarz method for solving linear system. This result correlates with the statements of \cite{needell2014stochastic}, but provides us with a new sense of similarity between SGD and Kaczmarz method.

\subsection{Binary logistic regression}
\subsubsection{Problem}
In this classification task then problem \eqref{strang:finitesum} takes the following form:
\begin{equation}\label{strang:LogReg}
-\frac{1}{n} \sum_{i=1}^n\left(y_i \ln \sigma(\vect{\theta}^\top\vect{x_i})  + (1-y_i) \ln (1-\sigma(\vect{\theta}^\top\vect{x_i}))\right) \to \min_{\vect{\theta} \in \mathbb{R}^p},
\end{equation}
where $\sigma(x) = \frac{1}{1 + e^{-x}}$ is the sigmoid function, while $ y_i \in \{0,1\}$ stands for the label of the object class.

\begin{equation}\label{strang:LogReg_grad}
\nabla_\theta f(\vect{\theta}) = \nabla f(\vect{\theta}) = \frac{1}{n}\sum\limits_{i=1}^n \vect{x_i}(\sigma(\vect{\theta}^\top\vect{x_i}) - y_i)
\end{equation}

The gradient flow equation will be written as follows:
\begin{equation}\label{strang:LogReg_GF}
\frac{d \vect{\theta}}{d t} = - \frac{1}{n}\sum\limits_{i=1}^n \vect{x_i}(\sigma(\vect{\theta}^\top\vect{x_i}) - y_i)
\end{equation}

Our particular interest lies in mini-batch reformulation of the given problem. We consider $s$ mini-batches with size $b$ regrouping ($b \cdot s = n$): $X_i \in \mathbb{R}^{b \times p}, \vect{y_i} \in \mathbb{R}^{b}$ and  $\sigma(\vect{x})$ stands for the element-wise sigmoid function.

\begin{equation}\label{strang:LogReg_GF_batch}
\frac{d \vect{\theta}}{d t} = - \frac{1}{n}\sum\limits_{i=1}^s X_i^\top\left(\sigma\left(X_i\vect{\theta}\right) - \vect{y_i}\right)
\end{equation}

\subsubsection{Splitting scheme and local problem}

Since we are applying splitting scheme to find the approximate solution of the \eqref{strang:LogReg_GF_batch}, each local problem should be written as follows:

\begin{equation}\label{strang:LogReg_GF_local}
\frac{d \vect{\theta}}{d t} =  - \frac{1}{n} X_i^\top\left(\sigma\left(X_i\vect{\theta}\right) - \vect{y_i}\right)
\end{equation}

Note, that this is not linear equation and cannot be solved as easy as in Theorem \ref{strang:LLS_local_solution}. However, we can apply the same technique to reduce the dimension of ODE, which is needed to be solved numerically. 

Suppose, we have $QR$ decomposition of each batch data matrix $X_i^\top = Q_i R_i$, then we can multiply both sides of \eqref{strang:LogReg_GF_local} on the $(I - Q_iQ_i^\top)$ on the left. 

\begin{align}
\nonumber (I - Q_iQ_i^\top)\frac{d \vect{\theta}}{d t} &= (I - Q_iQ_i^\top) \frac{1}{n} X_i^\top(\vect{y_i} -\sigma\left(X_i\vect{\theta}\right)) \\
\nonumber \frac{d \vect{\theta}}{d t} &= Q_i\frac{d (Q_i^\top\vect{\theta})}{d t} \quad Q_i^\top \vect{\theta} = \vect{\eta_i}\\ 
\nonumber \frac{d \vect{\theta}}{d t} &= Q_i\frac{d \vect{\eta_i}}{d t} \quad \text{integrate from $0$ to $h$}\\ 
\label{strang:logreg_theta_from_eta}\vect{\theta}(h) &= Q_i \left(\vect{\eta_i}(h) - \vect{\eta_i}(0) \right) + \vect{\theta}_0
\end{align}

On the other hand:
\begin{align}\nonumber
\frac{d \vect{\eta_i}}{d t} &= Q_i^\top\frac{d \vect{\theta}}{d t} =  - \frac{1}{n} Q_i^\top  X_i^\top(\sigma\left(X_i\vect{\theta}\right) - \vect{y_i}) = \\ 
\nonumber&= - \frac{1}{n} Q_i^\top  Q_i R_i(\sigma\left(X_i\vect{\theta}\right) - \vect{y_i}) =\\
\nonumber&= - \frac{1}{n} R_i(\sigma\left(X_i\vect{\theta}\right) - \vect{y_i})
\end{align}

Recall, that each hypothesis function depends on linear function $\vect{x_i}^\top \vect{\theta}$, which means, that in batch reformulation it is just entries of the vector $X_i \vect{\theta}$. Since we have $QR$ decomposition of $X_i^\top$, we can write: $X_i \vect{\theta} = R_i^\top Q_i^\top \vect{\theta} = R_i^\top \vect{\eta_i}$. In other words:

\begin{equation}
\label{strang:logreg_eta_ode}
\frac{d \vect{\eta_i}}{d t} = - \frac{1}{n} R_i\left(\sigma\left(R_i^\top \vect{\eta_i}\right) - \vect{y_i}\right),
\end{equation}

To sum it up, we need to solve \eqref{strang:logreg_eta_ode} (which is much simpler, than original differential equation \eqref{strang:LogReg_GF_local}), than substitute it to the
\eqref{strang:logreg_theta_from_eta} with $\vect{\eta_i}(0) = Q_i^\top \vect{\theta}_0$. Note, that matrices $Q_i$ and $R_i$ can be computed only once before the training.

\subsection{Softmax Regression}
\subsubsection{Problem}
In this classification task then problem \eqref{strang:finitesum} takes the following form:
\begin{equation}\label{strang:Softmax}
-\frac{1}{n} \sum_{i=1}^n\log\left(\frac{\vect{y_i}^\top e^{\Theta^\top \vect{x_i}}}{\vect{1}^\top e^{\Theta^\top \vect{x_i}}}\right) \to \min_{\Theta \in \mathbb{R}^{p \times K}},
\end{equation}
where $e^{\vect{x}}$ is element-wise exponential function, while $ \vect{y_i} \in \mathbb{R}^K$ stands for the one-hot encoding of the $i$-th object label.

\begin{equation}
\nabla_\Theta f(\Theta) = -\frac{1}{n} \sum_{i=1}^n\vect{x_i}\left(\vect{y_i} - \frac{ e^{\Theta^\top \vect{x_i}}}{\vect{1}^\top e^{\Theta^\top \vect{x_i}}}\right)^\top
\end{equation}

\begin{equation}
\nabla_\Theta f(\Theta) = -\frac{1}{n} \sum_{i=1}^n\vect{x_i}\left(\vect{y_i} - s\left(\Theta^\top \vect{x_i}\right)\right)^\top
\end{equation}


Here we use $s(\vect{x})$ as a softmax function of a vector $\vect{x}$, i.e. $s(\vect{x}) = \frac{e^{\vect{x}}}{\vect{1}^\top e^{\vect{x}}}$ .While mini-batch reformulation will take the following form:

\begin{equation}
\nabla_\Theta f(\Theta) = -\frac{1}{n} \sum_{i=1}^s X_i^\top\left(Y_i - s(\Theta^\top X_i^\top)\right)^\top,
\end{equation}

where $s(X) = \left[\begin{array}{cccc}| & | & | & | \\
s(\vect{x}_{(1)}) & s(\vect{x}_{(2)}) & \cdots & s(\vect{x}_{(b)}) \\
| & | & | & |
\end{array}\right]$ is a column-wise softmax function. Indeed, in a very similar manner to the binary logistic regression we can write down gradientflow ODE for softmax regression in a mini-batch form:

\begin{equation}
\frac{d \Theta}{d t} = - \frac{1}{n} \sum_{i=1}^s X_i^\top\left(s(\Theta^\top X_i^\top) - Y_i \right)^\top
\end{equation}

Splitting method requires the local problem, which is focused on a single minibatch:

\begin{equation}
\frac{d \Theta}{d t} = - \frac{1}{n} X_i^\top\left(s(\Theta^\top X_i^\top) - Y_i \right)^\top
\end{equation}

\begin{align}
\nonumber (I - Q_iQ_i^\top)\frac{d \Theta}{d t} &= (I - Q_iQ_i^\top) \frac{1}{n} X_i^\top(Y_i -s(\Theta^\top X_i^\top))^\top \\
\nonumber \frac{d \Theta}{d t} &= Q_i\frac{d (Q_i^\top\Theta)}{d t} \quad Q_i^\top \Theta = H_i \\ 
\nonumber \frac{d \Theta}{d t} &= Q_i\frac{d H_i }{d t} \quad \text{integrate from $0$ to $h$}\\ 
\label{strang:softmax_theta_from_eta}\Theta(h) &= Q_i \left(H_i (h) - H_i (0) \right) + \Theta_0
\end{align}

On the other hand:
\begin{align}\nonumber
\frac{d H_i }{d t} &= Q_i^\top\frac{d \Theta}{d t} =  - \frac{1}{n} Q_i^\top  X_i^\top(s(\Theta^\top X_i^\top) - Y_i)^\top = \\ 
\nonumber&= - \frac{1}{n} Q_i^\top  Q_i R_i(s(\Theta^\top X_i^\top) - Y_i)^\top =\\
\nonumber&= - \frac{1}{n} R_i(s(\Theta^\top X_i^\top) - Y_i)^\top =\\
\nonumber&= - \frac{1}{n} R_i(s(H_i^\top R) - Y_i)^\top 
\end{align}

Now we need to solve ODE of variable of the size $b \times k$, rather, than $p \times k$.

\end{document}